\documentclass[11pt]{article}
\usepackage{acl2015}
\usepackage{times}
\usepackage{url}
\usepackage{latexsym}
\usepackage{graphicx}
\usepackage{tikz-dependency}
\usepackage{tikz}
\usepackage{tikz-qtree}
\usepackage{amsmath}
\usepackage{amssymb}

\let\OLDthebibliography\thebibliography
\renewcommand\thebibliography[1]{
  \OLDthebibliography{#1}
  \setlength{\parskip}{2pt}
  \setlength{\itemsep}{0pt plus 0.3ex}
}

\usepackage{multirow}
\usepackage{algorithmic}
\usepackage{algorithm}
\usepackage{mathrsfs}
\usepackage{amsthm}
\usepackage{sidecap}
\usepackage{url}
\usepackage{etoolbox}

\AtBeginEnvironment{algorithmic}{\small}
\newenvironment{itemizesquish}{\begin{list}{\labelitemi}{\setlength{\itemsep}{0em}\setlength{\labelwidth}{0.5em}\setlength{\leftmargin}{\labelwidth}\addtolength{\leftmargin}{\labelsep}}}{\end{list}}

\newenvironment{enumeratesquish}{\begin{list}{\addtocounter{enumi}{1}\labelenumi}{\setlength{\itemsep}{0em}\setlength{\labelwidth}{0.5em}\setlength{\leftmargin}{\labelwidth}\addtolength{\leftmargin}{\labelsep}}}{\end{list}\setcounter{enumi}{0}}

\newcommand{\citet}[1]{\newcite{#1}}
\newcommand{\citep}[1]{\cite{#1}}

\newcommand{\sett}[1]{\mathcal{#1}}

\newcommand{\vectsymb}[1]{\boldsymbol{#1}}

\newcommand{\emphbf}[1]{{\textbf{#1}}}
\newcommand{\mathsc}[1]{\text{\sc{#1}}}

\newtheorem{theorem}{Theorem}

\newtheorem{proposition}[theorem]{\bf{Proposition}}

\newcommand{\afm}[1]{{\textcolor{red}{\bf [{\sc andre:} #1]}}}
\newcommand{\danny}[1]{{\textcolor{blue}{\bf [{\sc danny:} #1]}}}

\title{Parsing as \underline{R}eduction}
\author{Daniel Fern\'andez-Gonz\'alez$^\dagger$\thanks{ \ \ This research was carried out during an internship at Priberam Labs.}  \hspace{1cm} Andr\'e F. T. Martins$^\ddagger$$^\#$\\
$^\dagger$Departamento de Inform\'atica, Universidade de Vigo, Campus As Lagoas, 32004 Ourense, Spain\\
  $^\ddagger$Instituto de Telecomunica\c{c}\~oes, Instituto Superior T\'ecnico, 1049-001 Lisboa, Portugal \\
 $^\#$Priberam Labs, Alameda D. Afonso Henriques, 41, 2\textordmasculine, 1000-123 Lisboa, Portugal \\
  {\tt danifg@uvigo.es, atm@priberam.pt} } %\\\And
\date{}

\begin{document}
\maketitle
\begin{abstract}
We reduce phrase-representation parsing to dependency parsing. 
%By endowing dependency trees with additional order structure, we establish a formal 
%Key to our approach is 
Our reduction is grounded on
a new intermediate representation,  
``head-ordered dependency trees,'' 
%which we show 
shown 
to be
isomorphic to constituent trees. 
By encoding order information in the dependency labels, we show that any off-the-shelf, trainable dependency parser can be used to produce constituents. 
When this parser is non-projective, we can 
perform discontinuous parsing in a very natural manner.  
Despite the simplicity of our approach, 
experiments show that 
the resulting parsers are 
on par with strong baselines, such as the Berkeley parser for English and the best single system in 
the SPMRL-2014 shared task. % for eight morphologically rich languages. 
Results are particularly striking for discontinuous parsing of German, where we surpass the current state of the art by a wide margin. 
\end{abstract}

\section{Introduction}

\emphbf{Constituent parsing} %(also called phrase-based parsing) 
is a central problem in NLP---one at which statistical models trained on treebanks have excelled \citep{Charniak1996,Klein2003,Petrov2007NAACL}. 
However, most existing parsers are slow, 
since they need to deal with a heavy grammar constant. 
Dependency parsers are generally faster, but less informative, 
since they do not produce constituents, 
which are often required by downstream applications \citep{Johansson2008b,Wu2009,Kirkpatrick2011,Elming2013}. 
How to get the best of both worlds?

%\emphbf{Constituency parsing} is a central problem in NLP---one at which statistical models trained on treebanks have excelled 
%\citep{Charniak1996,Collins1999,Klein2003,Petrov2007NAACL}. 
%In %the last decade, 
%recent years, 
%dependency parsers emerged as a %popular and 
%faster alternative, 
%%While dependency-based syntactic formalisms emerged 
%%as a lightweight alternative, 
%%Despite the popularity of other syntactic formalisms, 
%%such as dependency syntax, 
%but they do not directly produce 
%constituents, 
%which are often required by downstream applications \citep{Johansson2008b,Wu2009,Kirkpatrick2011,Elming2013}. 
%%\citep{Wu2009,Kirkpatrick2011,Elming2013}. 
%%How to combine the advantages of the two formalisms? 
%How to get the best of both worlds?

%Making a constituency parser produce dependencies 
%is easy, by 
%using head percolation tables and simple transformation rules 
%\citep{Collins1999,Yamada2003,Marneffe2006}.  
Coarse-to-fine decoding \cite{Charniak2005} 
and shift-reduce parsing \citep{Sagae2005,Zhu2013} 
were a step forward to accelerate constituent parsing,  
% alleviate this bottleneck, 
but runtimes still lag %remain substantially slower 
%than 
those of dependency parsers.  
This is only made worse if \emphbf{discontinuous} constituents are allowed---such discontinuities are convenient to represent %topicalization, 
wh-movement, scrambling, extraposition, 
and other linguistic phenomena common in free word order languages. 
While 
non-projective dependency parsers, which are able to model such phenomena, have been widely developed 
in 
the last decade \citep{Nivre2007malt,McDonald2006CoNLL,Martins2013ACL}, discontinuous constituent parsing 
%with  structures %is by large an unsolved problem 
is still taking its first steps 
\citep{Maier2008,Kallmeyer2013}. 

%%Making a constituency parser produce dependencies 
%%is easy, by 
%%using head percolation tables and simple transformation rules 
%%\citep{Collins1999,Yamada2003,Marneffe2006}.  
%%this conversion is commonly more accurate than direct dependency parsing 
%%\citep{Charniak2005,Carreras2008}.   
%The major reason why constituency parsers are often eschewed in favor of  
%their dependency cousins  
%%So why have constituency parsers been partly overturned? The main reason 
%is \emph{speed}---%
%%the ones 
%%However, the constituency parsers 
%%developed so far are typically much slower than dependency parsers, 
%to be accurate, they typically   
%need to deal with a heavy grammar constant.  
%Coarse-to-fine decoders \cite{Charniak2005} 
%and shift-reduce parsers \citep{Sagae2005,Zhu2013} alleviate this bottleneck, 
%but the runtimes remain substantially slower 
%than those of dependency parsers.  
%This is only made worse if discontinuous constituents are allowed (these are convenient to represent %topicalization, 
%wh-movement, scrambling, extraposition, 
%and other linguistic phenomena common in free word order languages). 
%While 
%non-projective dependency parsers, which are able to model such phenomena, have been widely developed 
%in 
%the latest years, constituency parsing 
%with discontinuous structures is by large an unsolved problem 
%\citep{Maier2008,Kallmeyer2013}. 

In this paper, we show that an off-the-shelf, trainable, \emphbf{dependency parser}  
is enough   
%is all we need 
to build a highly-competitive constituent parser. 
%Suppose we are given a constituency treebank and an off-the-shelf, trainable, \emphbf{dependency parser}. 
%What can we do with it? 
%In this paper, we show that this is actually enough to devise a highly-competitive constituency parser.  
This (surprising) result is based on a  
%More precisely, what we propose is a simple 
\emphbf{reduction} of constituent %parsing 
to 
dependency parsing,  
%In this paper, we show that we can achieve top accuracies in data-driven constituency parsing by using only an off-the-shelf trainable dependency parser, 
followed by a simple 
post-processing procedure to recover unaries. 
Unlike other constituent parsers, ours does not 
require estimating a grammar,  %with production rules, 
nor binarizing the treebank.  
%nor a binarization of the treebank.  
%relying mostly on the %, but only 
%bilexical features used by 
%current dependency parsers. 
%It also does not require any binarization of the treebank. 
Moreover, when the dependency parser is non-projective, our method can perform discontinuous constituent parsing in a very natural way. %, 

Key to our approach is the notion of \emphbf{head-ordered 
dependency trees} (shown in Figure~\ref{fig:ctrees_dtrees}): by endowing dependency trees with this additional layer of structure, we show 
that they become isomorphic to constituent trees. 
%Simply put, decisions in the dependency tree about 
%which modifier to attach next to a given head are equivalent to decisions about 
%left-branching or right-branching in the 
%constituency tree (see Figure~\ref{fig:ctrees_dtrees}). 
We encode this structure as part of the 
dependency labels, enabling a dependency-to-constituent conversion. 
%Compared with the encoding strategy of \citet{Hall2008}, 
%we achieve a 10-fold decrease in the number of labels, which translates into better parsing accuracy. 
\citet{Hall2008} attempted a related conversion 
%A related conversion was attempted by 
%\citet{Hall2008} %, who 
to parse German, 
but their complex 
encoding scheme blows up 
the number of arc labels, 
%However, they encode 
%encode constituent spinal information in the dependency labels; 
% also encoded constituent information in the dependency labels, 
%much more complex and numerous dependency labels,  
%but their arc labels are much more complex and numerous, 
affecting 
%hurting
the final parser's quality. 
By contrast, our light encoding %Compared with their encoding strategy, we achieve 
achieves
a 10-fold decrease in the number of labels, 
%leading to more accurate parsing. 
translating into more accurate parsing. 
%better parsing accuracy. 
%We discuss related work in \S\ref{sec:related_work}. 

%Despite the simplicity of our approach, 
%experiments show 
%that 
%the resulting parsers 
While simple, our reduction-based parsers 
are on par with the Berkeley parser for English \citep{Petrov2007NAACL}, and with the best single system 
in the recent SPMRL shared task \citep{Seddah2014}, 
for eight morphologically rich languages. 
For discontinuous parsing, we surpass the 
current state of the art by a wide margin on two German datasets (TIGER and NEGRA), %several points, 
while achieving fast parsing speeds. 
Our parsers will be released along with this paper as accompanying software. 

%Talk about: Constituent parsers developed so far are quite slow in comparison with dependency parsers, that are able to handle more complex structures and faster (non-projective structures). This is specially crucial in languages with free-er order words and richer morphology such as German, that includes discontinuous phrase structures, where current constituent parsers are not even able to parse long sentences in a reasonable time. It seems a good option to handle the constituent problem with the efficient dependency parsers (as other authors did). So a phrase-dependency conversion is needed. The encoding algorithm is our main contribution. 

%\afm{dependency parsing generally faster than constituency parsing, due to the grammar constant}

%\afm{we avoid grammar binarization}

%\afm{opinion piece by \citet{Rambow2010} talking about representation of DP vs CP}

%\afm{While prior work aims a linguistically sound conversion \citep{Collins1999ACL,Xia2001} or joint treebank construction \citep{Xia2008}, which often involves handmade and language-specific transformation rules, we go back to the basics to provide a formal 
%connection between these two formalisms.}
%
%\afm{Spinal tree-adjoining grammar \citep{Carreras2008} and dual decomposition \citep{Rush2010}, two ways of tackling DP and CP jointly}

%\afm{Dependency-to-phrase conversion is also useful to downstream tasks \citep{Johansson2008b,Wu2009}}

\begin{figure*}[t]
\centering
\small
\begin{tikzpicture}[level distance=0.9cm]
\tikzset{frontier/.style={distance from root=100pt}}
\Tree [.S [.{NP} [.DT The ] [.\bf{NN} public ] ]
[.\bf{VP} [. \bf{VBZ} is ] [.ADVP [. \bf{RB} still ] ] [.ADJP [. \bf{JJ} cautious ] ] ] [.. . ] ]
\end{tikzpicture}
\qquad\qquad
\includegraphics[width=0.34\textwidth]{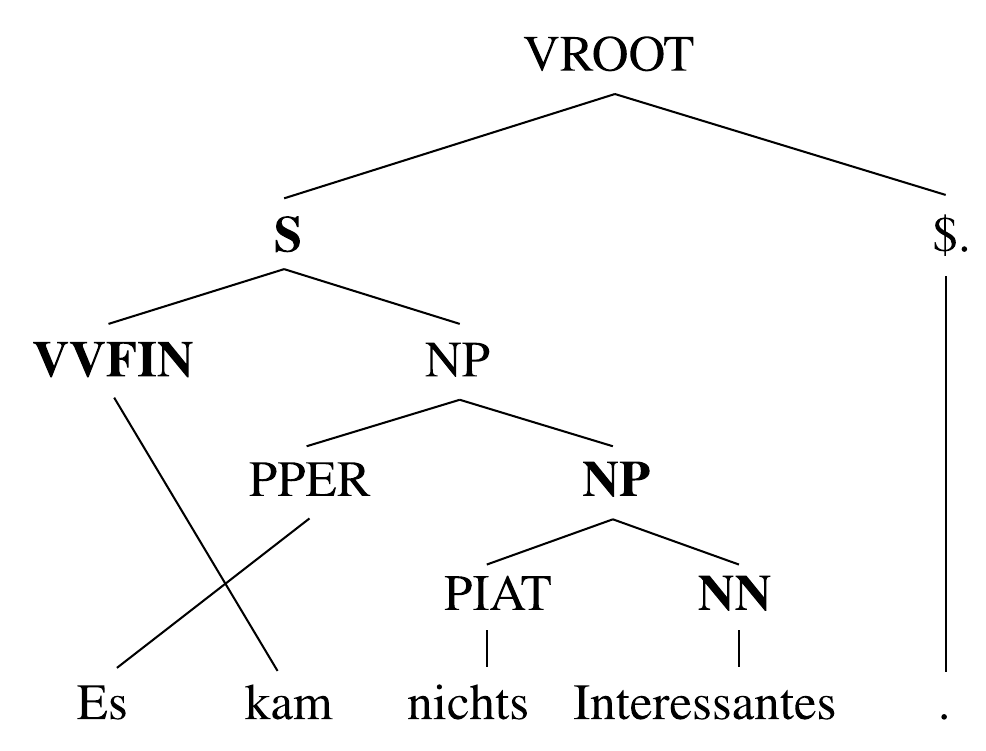}\\
\vspace{.3cm}
\begin{dependency}
\begin{deptext}[column sep=1.3em]
The \& public \& is \& still \& cautious \& . \\
{\small DT} \& {\small NN} \& {\small VBZ} \& {\small RB} \& {\small JJ} \& {\small .} \\
\end{deptext}
\depedge{2}{1}{NP\#1}
\depedge{3}{2}{S\#2}
\depedge{3}{4}{VP\#1}
\depedge{3}{5}{VP\#1}
\depedge{3}{6}{S\#2}
\end{dependency}
\qquad\qquad
\begin{dependency}
\begin{deptext}[column sep=1.3em]
Es \& kam \& nichts \& Interessantes \& . \\
{\small PPER} \& {\small VVFIN} \& {\small PIAT} \& {\small NN} \& {\small \$.} \\
\end{deptext}
\depedge[edge unit distance=2ex]{4}{1}{NP\#2}
\depedge{4}{3}{NP\#1}
\depedge[edge unit distance=4.5ex]{2}{4}{S\#1}
\depedge[edge unit distance=4ex]{2}{5}{VROOT\#1}
\end{dependency}
%\begin{tikzpicture}
%%\tikzset{frontier/.style={distance from root=120pt}}
%%\Tree [.S [.NP [.DT The ] [.NN finger-pointing ] ]
%%[.VP [.VBZ has ] [.ADVP [.RB already ] ] [.VP [.VBN begun ] ] ] [.. . ] ]
%\tikzset{every tree node/.style={align=center,anchor=north}}
%\Tree [.S [.NP DT\\The NN\\finger-pointing ]
%[.VP VBZ\\has [.ADVP RB\\already ] [.VP VBN\\begun ] ] {.}\\. ]
%\end{tikzpicture}
%\begin{dependency}
%\begin{deptext}[column sep=0.55em]
%The \& finger-pointing \& has \& already \& begun \& . \\
%{\tiny DT} \& {\tiny NN} \& {\tiny VBZ} \& {\tiny RB} \& {\tiny VBN} \& {\tiny .} \\
%\end{deptext}
%\depedge{2}{1}{NP\#1}
%\depedge{5}{2}{S\#2}
%\depedge{5}{3}{VP\#1}
%\depedge{5}{4}{VP\#1}
%\depedge{5}{6}{S\#2}
%\end{dependency}0
\caption{Top: a continuous (left) and a discontinuous (right) c-tree, taken from English PTB~\S{22} and 
German NEGRA, respectively. 
Head-child nodes are in bold. 
Bottom: corresponding head-ordered d-trees. 
The indices \#1, \#2, etc. denote the order of attachment events for each head. 
The English unary nodes {\tt ADVP} and {\tt ADJP} are dropped in the conversion.}
% Constituency tree for the sentence \#1581 in PTB~\S{22}, and equivalent head-ordered dependency representation. Right: Same for the sentence \#93 of the German NEGRA corpus, where the constituency tree is discontinuous and the dependency tree is non-projective.}
%\caption{Left: Constituency tree for the sentence \#1581 in PTB~\S{22}, and equivalent head-ordered dependency representation. Right: Same for the sentence \#93 of the German NEGRA corpus, where the constituency tree is discontinuous and the dependency tree is non-projective.}
\label{fig:ctrees_dtrees}
\end{figure*}

\section{Background}

We start by reviewing constituent and dependency representations, and setting up the notation. 
%Since constituencies and dependencies are a recurrent theme in this paper,  
Following \citet{Kong2014}, we use c-/d- prefixes for convenience (\emph{e.g.}, 
we write c-parser for constituent parser and d-tree for dependency tree). 

\subsection{Constituent Trees}

%In this paper, we consider a broad definition of constituency trees 
%(\emphbf{c-trees}) that subsumes 
%both phrase-representation trees, derived from a context-free grammar, as well as 
%trees with discontinuities, such as the ones derived with 
%linear context-free rewriting systems \citep{kallmeyer13}. 

Constituent-based representations %lie at the cornerstone of natural language syntax,   
%being 
are commonly seen as derivations according to a context-free grammar (CFG). Here, 
%in this paper, we abandon this perspective by 
%dispensing with the idea of an underlying grammar. Rather, we shift our attention to the constituency trees themselves. 
we focus on properties of the c-trees, rather than of the grammars used to generate them.  
We consider a broad scenario that permits c-trees with 
discontinuities, such as the ones derived with 
linear context-free rewriting systems (LCFRS; \citet{Vijay1987}). 
We also assume that the c-trees are 
lexicalized. 

Formally, let $w_1 w_2 \ldots w_L$ be a sentence, 
where $w_i$ denotes the word in the $i$th position. 
A \emphbf{c-tree}  %for this sentence 
is a rooted tree whose leaves are the words 
$\{w_i\}_{i=1}^L$, and whose 
internal nodes (constituents) %are constituents (called c-nodes). 
are represented 
as a tuple $\langle Z,h,\sett{I}\rangle$, 
where $Z$ is a non-terminal symbol, 
$h \in \{1,\ldots,L\}$ indicates the lexical head, 
and $\sett{I} \subseteq \{1,\ldots,L\}$ is the node's yield. %set of words 
Each word's parent is a pre-terminal unary node of the form 
$\langle p_i,i,\{i\} \rangle$,   
where $p_i$ denotes the word's part-of-speech (POS) tag. 
The yields and lexical heads are defined so that  
for every constituent 
$\langle Z,h,\sett{I}\rangle$ 
with children $\{\langle X_k, m_k, \sett{J}_k \rangle\}_{k=1}^K$, 
(i) we have 
$\sett{I} = \bigcup_{k=1}^K \sett{J}_k$; 
and (ii) there is a unique $k$ %(called a ``preferred child'') 
such that $h=m_k$. 
This $k$th node (called the head-child node) 
is commonly chosen 
applying an handwritten set of head rules %,  
%called a ``head percolation table'' 
\citep{Collins1999,Yamada2003}.

A c-tree is \emphbf{continuous} if 
all nodes $\langle Z,h,\sett{I}\rangle$ have a contiguous yield $\sett{I}$, and 
\emphbf{discontinuous} otherwise. 
Trees derived from a CFG are always continuous; 
those derived by a LCFRS may have discontinuities, 
the yield of a node being a union of spans, possibly with gaps in the middle. 
Figure~\ref{fig:ctrees_dtrees} shows an example of a continuous and a discontinuous c-tree. 
Discontinuous c-trees have crossing branches, 
if the leaves are drawn in left-to-right 
surface order. 
An internal node which is not a pre-terminal is called a \emphbf{proper node}. 
A node is called unary if it has exactly one child. 
A c-tree without unary proper nodes is called \emphbf{unaryless}. 
If all proper nodes have exactly two children then it is called a \emphbf{binary} c-tree.
%A \emphbf{binary} constituency tree is one in which every non-leaf node which is not a pre-terminal has exactly two children. 
Continuous binary trees may be regarded as having been generated by a CFG in Chomsky normal form. 

%\afm{The major reason why constituency parsers are often eschewed in favor of  
%their dependency cousins  
%%So why have constituency parsers been partly overturned? The main reason 
%is \emph{speed}---%
%%the ones 
%%However, the constituency parsers 
%%developed so far are typically much slower than dependency parsers, 
%to be accurate, they typically   
%need to deal with a heavy grammar constant.  
%Coarse-to-fine decoders \cite{Charniak2005} 
%and shift-reduce parsers \citep{Sagae2005,Zhu2013} alleviate this bottleneck, 
%but the runtimes remain substantially slower 
%than those of dependency parsers.}

\paragraph{Prior work.} 
%Since the creation of the Penn Treebank \citep{penn}, 
There has been a long string of work in statistical c-parsing, 
shifting from simple models \citep{Charniak1996} 
to more sophisticated ones 
%that propagate information 
%accross the tree %through techniques such as 
%via 
using 
structural annotation \citep{Johnson1998,Klein2003},
latent grammars \citep{Matsuzaki2005,Petrov2007NAACL}, 
and lexicalization \citep{Eisner1996,Collins1999}. 
%sometimes with discriminative training \citep{Taskar2004,Finkel2008}. 
%Shift-reduce parsers \citep{Sagae2005,Zhu2013} are another 
%line of work, usually leading to faster parsers than the ones based on global optimization. 
An orthogonal line of work uses ensemble or reranking strategies 
to further improve accuracy \citep{Charniak2005,Huang2008,Bjorkelund2014}.  
%Discontinuous trees are convenient 
%to represent phenomena such as topicalization, wh-movement, scrambling, and extraposition, being particularly suitable for free word order languages. 
%However, 
Discontinuous c-parsing is considered a much harder problem, involving mildly context-sensitive formalisms 
such as LCFRS or range concatenation grammars, 
with %\citep{Vijay1987,Kallmeyer2013}, 
treebank-derived c-parsers 
exhibiting near-exponential runtime  \cite[Figure~27]{Kallmeyer2013}. 
To speed up decoding, prior work has considered restrictons,   
%additional 
such as 
bounding the fan-out  \citep{Maier2012} and requiring well-nestedness \citep{Kuhlmann2006,Rodriguez2010}.  
Other approaches eliminate the discontinuities via tree transformations \citep{Boyd2007,Kubler2008}, 
sometimes as a pruning step followed by reranking \citep{Cranenburgh2013}.   
However, reported runtimes are still superior to 10 seconds per sentence, which is not practical. 
Recently, \citet{Versley2014SPMRL} proposed an easy-first approach 
%that is much faster 
%but less accurate. 
that leads to considerable speed-ups, 
but is less accurate. 
%%leading to 40--55 sentences per second. 
%with a sacrifice 
In this paper, we design fast discontinuous c-parsers that outperform all the ones above by a wide margin, with similar runtimes as \citet{Versley2014SPMRL}. %\afm{check} %, 

\subsection{Dependency Trees}

%Dependency parsers are the main driving force behind our approach to 
%parse constituents. 
In this paper, we use d-parsers as a black box to parse constituents. 
Given a sentence $w_1 \ldots w_L$, 
%where $w_i$ denotes the word in the $i$th position, 
%their goal is to produce a \emphbf{d-tree} representation---%
a \emphbf{d-tree} is 
a directed tree
spanning all the words in the sentence.%
\footnote{We assume throughout that dependency trees have a single root among $\{w_1,\ldots,w_L\}$. Therefore, there is no need to consider an extra root symbol, as often done in the literature.} %
Each arc in this tree %(a d-arc) 
is a tuple $\langle h, m, \ell \rangle$, 
expressing a typed dependency relation $\ell$ %\in \sett{L}$ 
between the head word $w_h$ and the modifier $w_m$. 
%(here, $\sett{L}$ is a predefined set of dependency labels).

A d-tree is  
\emphbf{projective} if for every arc $\langle h, m, \ell \rangle$  
there is a directed path from $h$ to all words that lie between $h$ and $m$ in the surface string \citep{Kahane1998}. 
Projective d-trees can be obtained from %lexicalizing 
continuous c-trees by reading off the lexical heads and dropping the internal nodes \citep{Gaifman1965}. 
However, this relation is many-to-one: as shown in Figure~\ref{fig:ctrees_ambiguity}, several c-trees 
may project onto the same d-tree, differing on their flatness and on  
left or right-branching decisions. 
In the next section, we introduce the concept of 
head-ordered d-trees and express one-to-one mappings between 
these two representations.

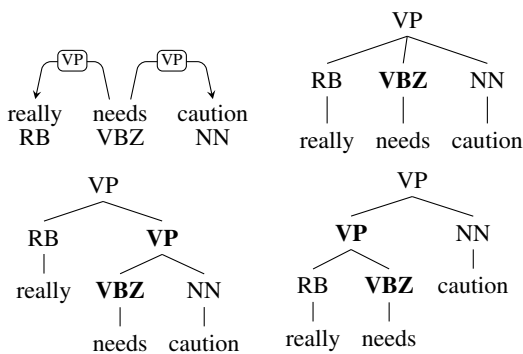
\begin{figure}[t]
\small
\centering
\begin{dependency}
\begin{deptext}[column sep=0.8em]
really \& needs \& caution \\
{\small RB} \& {\small VBZ} \& {\small NN} \\
\end{deptext}
\depedge{2}{1}{VP}
\depedge{2}{3}{VP}
\end{dependency}
\quad
\begin{tikzpicture}[level distance=0.8cm]
\tikzset{frontier/.style={distance from root=40pt}}
\Tree [.VP [. RB really ] [. \bf{VBZ} needs ]  [. NN caution ] ]
\end{tikzpicture}
\begin{tikzpicture}[level distance=0.7cm]
\tikzset{frontier/.style={distance from root=57pt}}
\Tree [.VP [. RB really ] [. \bf{VP} [. \bf{VBZ} needs ] [. NN caution ] ] ] 
\end{tikzpicture}
\quad
\begin{tikzpicture}[level distance=0.7cm]
\tikzset{frontier/.style={distance from root=57pt}}
\Tree [.VP [. \bf{VP} [. RB really ] [. \bf{VBZ} needs ] ] [. NN caution ] ] 
\end{tikzpicture}
\caption{Three different c-structures for the VP ``\emph{really needs caution}.'' All are consistent with the d-structure at the top left.}
\label{fig:ctrees_ambiguity}
\end{figure}

\paragraph{Prior work.} 
There has been a considerable amount of work developing  rich-feature d-parsers. 
While projective d-parsers can use dynamic programming \citep{Eisner1999,Koo2010}, 
non-projective d-parsers typically rely on approximate decoders, since 
the underlying problem is NP-hard beyond 
arc-factored models \citep{McDonald2007}. 
An alternative are transition-based d-parsers \citep{Nivre2006CoNLL,Zhang2011}, 
which achieve observed linear time. 
Since d-parsing algorithms do not have a grammar constant, typical implementations are significantly faster than c-parsers \citep{Rush2012,Martins2013ACL}. 
The key contribution of this paper is to reduce c-parsing to 
d-parsing, allowing to bring these runtimes closer.

\section{Head-Ordered Dependency Trees}

We next endow d-trees with another layer of structure, namely \emphbf{order information}. 
In this framework, not all modifiers of a head are ``born equal.'' 
Instead, their attachment to the head occurs as a sequence of ``events,'' which reflect the head's preference for attaching some modifiers before others. 
As we will see, this additional structure will undo the ambiguity expressed in Figure~\ref{fig:ctrees_ambiguity}.

\subsection{Strictly Ordered Dependency Trees}\label{sec:strictly_ordered_dtree}

%As a warm-up, consider 
Let us start with 
the simpler case where the attachment order is strict.  
For each head word $h$ with modifiers $M_h = \{m_1,\ldots,m_K\}$, 
we endow $M_h$ with a \emphbf{strict order relation} $\prec_h$, 
so we can 
organize all %its elements 
the modifiers of $h$ 
as a chain, 
$m_{i_1} \prec_h m_{i_2} \prec_h \ldots \prec_h m_{i_K}$.  
We regard 
this chain as reflecting the order by 
which words are attached (\emph{i.e.}, if $m_i \prec_h m_j$ this means that 
``$m_i$ is attached to $h$ before $m_j$''). 
%In Figure~\ref{alg:dtree_to_ctree_conversion}, 
We represent this graphically by 
decorating d-arcs with indices ($\#1, \#2, \ldots$) to denote the order of events, 
as we do in Figure~\ref{fig:ctrees_dtrees}. 

\if 0
We say that the strict order $\prec_h$ has the \emphbf{nesting property} iff 
closer words in the same direction are always attached first, \emph{i.e.}, iff 
$h < m_i < m_j$ or $h > m_i > m_j$ implies 
$m_i \prec_h m_j$. 
A projective dependency tree endowed with a nested strict order for each head is called a ``nested-strictly ordered projective dependency tree'' (which we abbreviate to \emphbf{NSOP d-tree}). 

Next, we show that NSOP d-trees are in a one-to-one correspondence with 
binary continuous constituency trees. 
The intuition is simple: if $\prec_h$ has the nesting property, 
then, at each point in time, all the head needs to decide about the next event is whether to attach the closest available modifier on the \emph{left} or on the \emph{right}. 
This is similar to choosing between left-branching or right-branching in a 
constituency tree. 
\fi

A d-tree endowed with a strict order for each head is called a \emphbf{strictly ordered d-tree}.  
We establish below a correspondence between 
strictly ordered d-trees and 
binary c-trees. 
Before doing so, we need a few more definitions about c-trees. 
For each word position $h \in \{1,\ldots,L\}$, we 
define $\psi(h)$ as the node higher in the c-tree whose lexical head is 
$h$. We call the path from $\psi(h)$ down to the pre-terminal $p_h$ the \emphbf{spine} of $h$. We may regard a c-tree as a set of $L$ spines, one per word, 
which attach to each other to form a tree \citep{Carreras2008}. 
%The intuition is simple: %if $\prec_h$ has the nesting property, 
%at each point in time, the head picks a modifier either on the left or on the right, and attaches it to its spine. 
%This corresponds to left-branching or right-branching in the 
%constituency tree. 
%\afm{For each word $h$, we define $\psi(h)$ as the c-node higher in the tree whose head is $h$. We consider the vertical spine of $h$ as the set of nodes in the path from $h$ to $\psi(h)$; the cardinality of this set gives the height of $\psi(h)$ in the tree. We may associate to each node in the spine, with a height $j$, a decision of which modifier to attach. Let $Z$ be the label of this node, $j$ its height, and $m$ the attached modifier at that level, we then create a d-arc $\langle h, m, Z \rangle$ with an index $j$.} 
%
%\afm{Finally, for each word $w_h$, we define $\psi(h)$ as the highest node $v$ with 
%$h = h(v)$. The path from $w_h$ to $\psi(h)$ is called the \emphbf{spine} of $h$; 
%we may regard a c-tree as a set of spines attached to each other, as 
%Figure~\ref{} illustrates.}
%\subsection{Equivalence Results}
%\afm{maybe we need to define a few more things here, such as $\psi(h)$, the spine of a c-tree, etc.}
%
%\afm{add a figure with a spine which shows how a c-tree can be transformed into a d-tree and that transformation is invertible}
\begin{figure}[t]
\centering
\includegraphics[width=0.9\columnwidth]{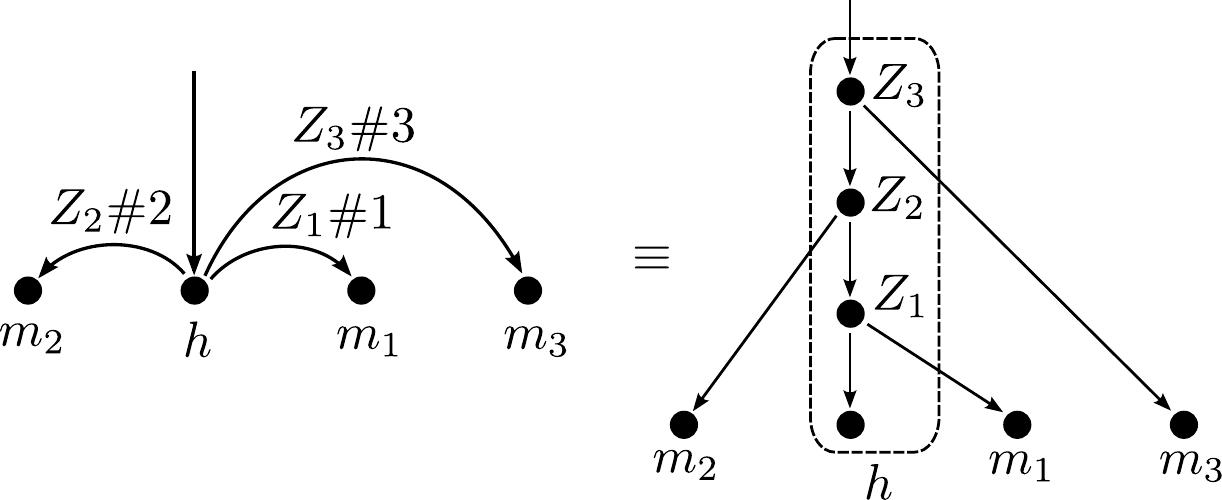} 
\caption{Transformation of a strictly-ordered d-tree into a binary c-tree.}
\label{fig:dtree_ctree_conversion}
\end{figure}
We then have the following
\begin{proposition}\label{prop:strictorder}
%There is a one-to-one correspondence between 
Binary c-trees and strictly-ordered d-trees are isomorphic, i.e., there is a one-to-one correspondence between the two sets,  
where the number of symbols is preserved. 
%where the set of non-pre-terminal symbols $\sett{N}$ equals the set of dependency labels 
%$\sett{L}$. 
\end{proposition}
%\begin{proposition}\label{prop:strictorder}
%The set of binary c-trees over $\langle \sett{W}, \sett{P}, \sett{N} \rangle$ and the set of strictly-ordered d-trees over $\langle \sett{W}, \sett{P}, \sett{N} \rangle$ are isomorphic, i.e., there is a one-to-one correspondence between the two sets. 
%\end{proposition}

%\begin{proposition}\label{prop:strictorder}
%There is a one-to-one correspondence between binary continuous c-trees and NSOP d-trees.
%\end{proposition}
\begin{proof}%[Proof sketch]
We use the construction in Figure~\ref{fig:dtree_ctree_conversion}. 
We will show that, given an arbitrary strictly-ordered d-tree $\sett{D}$, 
we can perform an invertible transformation to turn it into a binary c-tree $\sett{C}$; 
and vice-versa.  
Let $\sett{D}$ be given. 
We visit each node $h\in \{1,\ldots,L\}$ and split it into $K+1$ nodes, where $K=|M_h|$, organized as 
a linked list, as 
%the construction in 
Figure~\ref{fig:dtree_ctree_conversion} illustrates  
(this will become the spine of $h$ in the c-tree). 
%the parent of $h$ in the d-tree will move its pointer to the head of the list. 
%, and the tail of the list will become a pre-terminal node. 
For each modifier $m_k \in M_h$ with $m_1 \prec_h \ldots \prec_h m_K$, move the tail of the arc $\langle h,m_k,Z_k\rangle$ to the $(K+1-k)$th node of the linked list and assign the label $Z_k$ to this node, letting $h$ be its lexical head. 
%This construction is illustrated in Figure~\ref{fig:dtree_ctree_conversion}. 
Since the incoming and outgoing arcs of the linked list component are the same as in the original node $h$, the tree structure is preserved.  
After doing this for every $h$, add the leaves 
and propagate the yields bottom up. 
It is straightforward to show that this procedure yields a valid binary c-tree. Since there is no loss of information (the orders $\prec_h$ are implied by the order of the nodes in each spine), this construction can be inverted to recover the original d-tree. 
Conversely, if we start with a binary c-tree, traverse the spine of each $h$, and attach the modifiers $m_1 \prec_h \ldots \prec_h m_K$ in order, we get a strictly ordered d-tree (also an invertible procedure).
\end{proof}

\subsection{Weakly Ordered Dependency Trees}\label{sec:weakly_ordered_dtree}

Next, we relax the strict order assumption, restricting the modifier sets 
$M_h = \{m_1,\ldots,m_K\}$ to be only \emphbf{weakly ordered}. 
This means that we can partition the $K$ modifiers into $J$ 
equivalence classes,  
$M_h = \bigcup_{j=1}^J \bar{M}_h^j$, 
and define a strict order $\prec_h$ on the 
quotient set: $\bar{M}_h^1 \prec_h \ldots \prec_h \bar{M}_h^J$. 
Intuitively, there is still a sequence of events 
($1$ to $J$), but now at each event $j$ it may happen that multiple 
modifiers (the ones in the equivalence set $\bar{M}_h^j$)
are simultaneously attached to $h$. 
% corresponds to an attachment preference  where there can be ties.
%Next, we relax the strict order assumption on d-trees. 
%For each head $h$, we assume a \emphbf{weak order} $\precsim_h$ on $M_h$.% 
%\footnote{That is, $\precsim_h$ is reflexive ($m_i \precsim_h m_i$), 
%transitive ($m_i \precsim_h m_j \wedge m_j \precsim_h m_k \Rightarrow m_i \precsim_h m_k$), 
%and total (for any pair $m_i,m_j \in M_h$, either $m_i \precsim_h m_j$, or $m_j \precsim_h m_i$, or both).} %
%This induces an equivalent relation $\approx_h$ on $M_h$:  
%$m_i \approx_h m_j$ iff 
%$m_i \precsim_h m_j$ and $m_j \precsim_h m_i$. 
%Intuitively,  this corresponds to the case where $m_i$ and $m_j$ are attached 
%\emphbf{simultaneously} to $h$.
%Similarly to above, we say that $\prec_h$ has the nesting property iff  
%$h < m_i < m_j$ or $h > m_i > m_j$ implies 
%%$m_i \precsim_h m_j$;  
%that $m_i$ and $m_j$ either belong the same equivalent class (written $m_i \equiv_{h} m_j$), 
%or $m_i \prec_h m_j$. 
%A projective dependency tree endowed with a nested weak order for each head, 
%satisfying the additional property that $m_i \equiv_h m_j$ implies that 
%the dependency labels of $\langle h, m_i\rangle$ and $\langle h, m_j\rangle$ are the same, 
%is called a \emphbf{nested-weakly ordered projective dependency tree}, abbreviated to NWOP d-tree. 
A \emphbf{weakly ordered d-tree} is a d-tree endowed with a 
weak order for each head 
and such that any pair $m,m'$ in the same equivalence class (written $m \equiv_h m'$) receive the same  
dependency label $\ell$. %of $\langle h, m_i, \ell\rangle$ and $\langle h, m_j, \ell\rangle$ are the same, 
%is called a \emphbf{weakly ordered d-tree}. 
%When convenient, we denote arcs in this tree as tuples $\langle h,m,\ell,j \rangle$, 
%where $j \in \{1,\ldots,J\}$ denotes the ``index'' of the equivalent class $\bar{M}_h^j$.

We now show that Proposition~\ref{prop:strictorder} can be generalized 
to weakly ordered d-trees. 
\begin{proposition}\label{prop:weakorder}
Unaryless c-trees %
%\footnote{An internal c-node which not a pre-terminal is called a proper node. A node with only one child is called a unary node. A \emphbf{unaryless tree} is one in which 
%there are no unary proper nodes.} %
and weakly-ordered d-trees are isomorphic. 
\end{proposition}
%\begin{proposition}\label{prop:weakorder}
%There is a one-to-one correspondence between continuous c-trees without unary nodes%
%\footnote{A node is called unary if it has a single child.} %
%and NWOP d-trees. 
%\end{proposition}
\begin{proof}%[Proof sketch]
This is a simple extension of Proposition~\ref{prop:strictorder}. The construction is the same as in Figure~\ref{fig:dtree_ctree_conversion}, but now we can collapse some 
of the nodes in the linked list, originating more than one modifier attaching to the same position of the spine---this is only possible for sibling arcs with the same index and the same arc label. 
Note, however, that if we started with a c-tree with unary nodes and tried to invert this procedure to obtain a d-tree, the unary nodes would be lost, since they do not involve attachment of modifiers. In a chain of unary nodes, only the last node would be recovered when inverting this transformation. 
\end{proof}
We emphasize that Propositions~\ref{prop:strictorder}--\ref{prop:weakorder} hold 
without blowing up the number of symbols. That is, the dependency label alphabet 
is exactly the same as the set of phrasal symbols in the constituent representations. 
%
%\subsection{Conversion Algorithms}
%
Algorithms~\ref{alg:ctree_to_dtree_conversion}--\ref{alg:dtree_to_ctree_conversion} 
convert back and forth between the two 
formalisms, performing the construction of 
Figure~\ref{fig:dtree_ctree_conversion}. 
Both algorithms run in linear time 
with respect to the size of the sentence.

\begin{algorithm}[t]
\small{%
   \caption{Conversion from c-tree to d-tree} \label{alg:ctree_to_dtree_conversion}
\begin{algorithmic}[1]
\REQUIRE c-tree  $\sett{C}$.
\ENSURE head-ordered d-tree $\sett{D}$. 
   %\STATE $\sett{C} \leftarrow \varnothing$
   \STATE ${\mathrm{Nodes}} := \mathsc{GetPostOrderTraversal}(\sett{C})$.\label{alg:ctree_to_dtree_conversion_postorder}
   \STATE Set $j(h) := 1$ for every $h = 1,\ldots,L$. 
   \FOR{$v := \langle Z, h, \sett{I} \rangle \in \mathrm{Nodes}$}
     %\STATE Let $Z$ be the label of $v$, and $u^*$ its preferred child.
     %\STATE Set $h := h(u^*)$, $h(v) := h$, and $\bar{M}_{h}^{j(h)}(\sett{D}):=\varnothing$.
     \FOR{every $u := \langle X, m, \sett{J}\rangle$ which is a child of $v$}
     %\FOR{$u \in \sett{U} \setminus \{u^*\}$}
       \IF{$m \ne h$}
         \STATE Add to $\sett{D}$ an arc $\langle h, m, Z \rangle$, and put it in %the $j(h)$th equivalence class 
         $\bar{M}_h^{j(h)}$.\label{alg:ctree_to_dtree_conversion_arc_creation}
       \ENDIF
       %\STATE Add $m$ to the equivalent class $\bar{M}_{h}^{j(h)}(\sett{D})$.
     \ENDFOR
     \STATE Set $j(h) := j(h)+1$.
   \ENDFOR
\end{algorithmic}}
\end{algorithm}

\begin{algorithm}[t]
\small{%
   \caption{Conversion from d-tree to c-tree} \label{alg:dtree_to_ctree_conversion}
\begin{algorithmic}[1]
\REQUIRE head-ordered d-tree  $\sett{D}$.
\ENSURE c-tree $\sett{C}$. %=\langle \sett{C}_n, \sett{C}_e\rangle$. 
   %\STATE $\sett{C} \leftarrow \varnothing$
   \STATE ${\mathrm{Nodes}} := \mathsc{GetPostOrderTraversal}(\sett{D})$.\label{alg:dtree_to_ctree_conversion_postorder}
   \FOR{$h \in \mathrm{Nodes}$}
     \STATE %{\bfseries Base case:} 
     Create %pre-terminal 
     $v:=\langle p_h, h, \{h\}\rangle$ %, connect it to leaf $w_h$, 
     and set $\psi(h) := v$.\label{alg:dtree_to_ctree_conversion_basecase}
     %\STATE {\bfseries Base case:} $\psi(h) := \langle t_h, h, \{h\} \rangle$\label{alg:dtree_to_ctree_conversion_basecase}
     \STATE Sort $M_h(\sett{D})$, 
     %according to $\prec_h$, 
     yielding %a chain of equivalence classes 
     $\bar{M}_h^1\prec_h \bar{M}_h^2\prec_h \ldots \prec_h \bar{M}_h^J$.
     \FOR{$j=1,\ldots,J$}
       \STATE Let $Z$ be the label in %common label of arcs 
       $\{\langle h,m,Z \rangle\,\,|\,\, m \in \bar{M}_h^j\}$.
%       \STATE  
%       $\psi(h) = \langle X,h,\sett{I}\rangle$ and $\psi(m) = \langle Y_m,m,\sett{J}_m\rangle$, $\forall m \in  \bar{M}_h^j$.
       \STATE Obtain c-nodes 
       $\psi(h) = \langle X,h,\sett{I}\rangle$ and $\psi(m) = \langle Y_m,m,\sett{J}_m\rangle$ for all $m \in  \bar{M}_h^j$.
       %\STATE Create c-node $v := \langle Z, \{\psi(h)\} \cup \{\psi(m_j)\}_{m_j \in \bar{M}_h^j}, \psi(h) \rangle$ and add $v$ to $\sett{C}$.
       \STATE Add c-node $v := \langle Z, h, \sett{I} \cup \bigcup_{m \in  \bar{M}_h^j} \sett{J}_m\rangle$ %and add it 
       to $\sett{C}$.\label{alg:dtree_to_ctree_conversion_cnode_creation}
       \STATE Set $\psi(h)$ and 
       %$\{\psi(m)\}_{m \in \bar{M}_h^j}$ as children of $v$.
       $\{\psi(m) \,|\,m \in \bar{M}_h^j\}$ as children of $v$.\label{alg:dtree_to_ctree_conversion_cnode_children}
       \STATE Set $\psi(h) := v$.
       %\STATE Obtain c-nodes $\psi(h) = \langle X,h,\sett{I}\rangle$ and $\psi(m_j) = \langle Y,m_j,\sett{J}_j\rangle$ for all $m_j \in  \bar{M}_i$
       %\STATE Create c-node $\psi'(h) = \langle Z,h,\sett{I} \cup \bigcup_{m_j \in  \bar{M}_i} \sett{J}_j)\rangle$
       %\STATE Add node $\sett{C}_n \leftarrow \sett{C}_n \cup \{\psi'(h)\}$
       %\STATE Add edges $\sett{C}_e \leftarrow \sett{C}_e \cup \{\psi'(h) \rightarrow \psi(h)\} \cup \bigcup_{m_j \in  \bar{M}_i} \{\psi'(h) \rightarrow \psi(m_j)\}$
     \ENDFOR
   \ENDFOR
\end{algorithmic}}
\end{algorithm}

\subsection{Continuous and Projective Trees} 
%The Projective and Nested Case}

%Propositions~\ref{prop:strictorder}--\ref{prop:weakorder} both apply to 
%general trees, which can be non-projective or discontinuous. 
What about the more restricted class of projective d-trees? Can we find an equivalence relation with continuous c-trees? 
%What about non-projective dependency trees? Can we find an equivalence relation with discontinuous constituency trees? 
In this section, we give a precise answer to this question. 
It turns out that we need an additional property, 
%We now show that continuous c-trees are equivalent to d-trees with the following property, 
illustrated in Figure~\ref{fig:nonproj_nonnested_discontinuous}.  
%We show below that discontinuous c-trees can arise not 
%only from non-projective d-trees, but also 
%from projective ones which violate  
%a property that we next define and is illustrated in Figure~\ref{fig:nonproj_nonnested_discontinuous}. 
%\afm{We say that the strict order $\prec_h$ has the \emphbf{nesting property} iff 
%closer words in the same direction are always attached first, \emph{i.e.}, iff 
%$h < m_i < m_j$ or $h > m_i > m_j$ implies 
%$m_i \prec_h m_j$. 
%A projective dependency tree endowed with a nested strict order for each head is called a ``nested-strictly ordered projective dependency tree'' (which we abbreviate to \emphbf{NSOP d-tree}).} 

We say that $\prec_h$ has the \emphbf{nesting property} iff  
closer words in the same direction are always attached first, 
\emph{i.e.}, iff
$h < m_i < m_j$ or $h > m_i > m_j$ implies 
that either $m_i \equiv_{h} m_j$ 
%$m_i \precsim_h m_j$;  
%that $m_i$ and $m_j$ either belong the same equivalent class (written $m_i \equiv_{h} m_j$), 
or $m_i \prec_h m_j$. 
A weakly-ordered d-tree which is projective and whose orders $\prec_h$ have the nesting property for every $h$ 
is called a \emphbf{nested-weakly ordered projective d-tree}. 
%, abbreviated to NWOP d-tree. 
\begin{figure}[t]
\centering
\includegraphics[width=0.9\columnwidth]{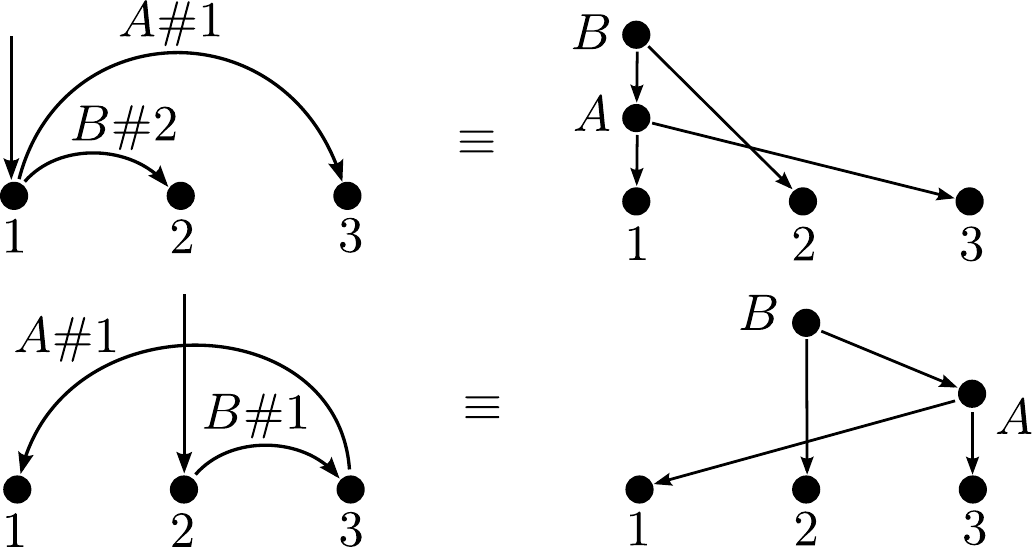} 
\caption{Two discontinuous constructions caused by a 
non-nested order (top) and a non-projective d-tree (bottom). In both cases node $A$ has a non-contiguous yield.}
\label{fig:nonproj_nonnested_discontinuous}
\end{figure}
%\afm{Talk about weakly ordered nonprojective dependency trees (WONP d-trees) and discontinuous c-trees. Note: we need to define ``nonprojective'' and ``discontinuous'' as broader, but not opposed to, 
%projective and continuous.} 
We then have the following result. 
%, which we may regard as an extension of 
%the relation established by \citet{Gaifman1965} 
%between projective d-trees and continuous c-trees. 

%\begin{proposition}\label{prop:discontinuous_case}
%There is a one-to-one correspondence between unary-less c-trees (possibly discontinuous) %without unary nodes 
%and weakly ordered d-trees (possibly non-projective and without the nesting property). 
%%There is a one-to-one correspondence between discontinuous c-trees without unary rules and WONP d-trees. 
%\end{proposition}
\begin{proposition}\label{prop:continuous_case}
Continuous unaryless c-trees  
and nested-weakly ordered projective d-trees are isomorphic. 
%There is a one-to-one correspondence between discontinuous c-trees without unary rules and WONP d-trees. 
\end{proposition}
\begin{proof}%[Proof sketch]
We need to show that (i) Algorithm~\ref{alg:ctree_to_dtree_conversion}, when applied to a continuous c-tree $\sett{C}$, retrieves a head ordered d-tree $\sett{D}$ which is projective and has the nesting property, (ii) vice-versa for Algorithm~\ref{alg:dtree_to_ctree_conversion}. 
To see (i), note that the projectiveness of $\sett{D}$ is ensured by the well-known result of \citet{Gaifman1965} about the projection of continuous trees. To show that it satisfies the nesting property, note that 
nodes higher in the spine of a word $h$ are always attached by modifiers farther apart (otherwise edges in $\sett{C}$ would cross, which cannot happen for a continuous $\sett{C}$). 
To prove (ii), we use induction. We need to show that every created c-node in Algorithm~\ref{alg:dtree_to_ctree_conversion} has a contiguous span as yield. The base case (line~\ref{alg:dtree_to_ctree_conversion_basecase})  is trivial. Therefore, it suffices to show that 
in line~\ref{alg:dtree_to_ctree_conversion_cnode_creation}, assuming the yields of (the current) $\psi(h)$ and each $\psi(m)$ are contiguous spans, the union of these yields is also contiguous. 
Consider the node $v$ when these children have
been appended  
(line~\ref{alg:dtree_to_ctree_conversion_cnode_children}),  
%denote the current $\psi(h)$ node, 
and 
choose $m \in \bar{M}_h^j$ arbitrarily. We only need to show that for any $d$ between $h$ and $m$, $d$ belongs to the yield of $v$. 
Since $\sett{D}$ is projective and there is a d-arc between $h$ and $m$, we have that $d$ must descend from $h$. Furthermore, since projective trees cannot have crossing edges, we have that $h$ has a unique child $a$, also between $h$ and $m$, which is an ancestor of $d$ (or $d$ itself). Since $a$ is between $h$ and $m$, from the nesting property, 
we must have $\langle h,m,\ell\rangle \not\prec_h \langle h,a,\ell'\rangle$  
%the  equivalence class index of the arc $h \rightarrow a$ cannot exceed that of $h \rightarrow m$. 
Therefore, 
since we are processing the modifiers in order, 
we have that $\psi(a)$ is already a descendent of $v$ 
after line~\ref{alg:dtree_to_ctree_conversion_cnode_children}, which implies that the yield of $\psi(a)$ (which must include $d$, since $d$ descends from $a$) must be contained in the yield of $v$. 
\end{proof}

Together, Propositions~\ref{prop:strictorder}--\ref{prop:continuous_case} have as corollary that 
nested-strictly ordered projective d-trees are in a one-to-one correspondence with 
binary continuous c-trees. 
The intuition is simple: if $\prec_h$ has the nesting property, 
then, at each point in time, all one needs to decide about the next event is whether to attach the closest available modifier on the \emph{left} or on the \emph{right}. 
This corresponds to choosing between left-branching or right-branching in a 
c-tree. 
While this is potentially interesting for most continuous c-parsers, which work with binarized c-trees  
when running the CKY algorithm, 
our c-parsers (to be described in \S\ref{sec:our_constituency_parsers}) 
do not require any binarization since they 
work with weakly-ordered d-trees, using Proposition~\ref{prop:weakorder}. 

\section{Reduction-Based Constituent Parsers}\label{sec:our_constituency_parsers}

We next show how to use the equivalence results 
obtained in the previous section to 
design c-parsers when only a 
trainable d-parser is available. 
%The key idea is to reduce 
%constituency parsing to dependency parsing, followed by 
%simple post-processing. 

%We present a straightforward algorithm that does not lose any information throughout the constituent-to-dependency encoding and the inversely dependency-to-constituent transformation. This makes possible to parse a constituency tree with any dependency parser under the same time complexity as in a dependency parsing scenario.

Given a c-treebank provided as input, our procedure is outlined as follows:
\begin{enumeratesquish}
\item Convert the c-treebank to dependencies (Algorithm~\ref{alg:ctree_to_dtree_conversion}). 
\item Train a labeled d-parser on this treebank. 
\item For each test sentence, run the labeled d-parser and convert the predicted d-tree into a c-tree without unary nodes (Algorithm~\ref{alg:dtree_to_ctree_conversion}). 
\item Do post-processing to recover unaries.
\end{enumeratesquish}
The next subsections describe each of these steps. %in detail. 
Along the way, we illustrate with 
experiments using %statistics collected from 
the English Penn Treebank \citep{penn}, 
which we lexicalized by applying the head rules  of \citet{Collins1999}.%
\footnote{We train on \S{02--21}, use \S{22} for validation, and test on \S{23}. We predict automatic POS tags with \emph{TurboTagger}~\citep{Martins2013ACL}, with 10-fold jackknifing on the training set.} %

\subsection{Dependency Encoding}\label{sec:encoding}

The first step is to convert the c-treebank to head-ordered dependencies, which we do using Algorithm~\ref{alg:ctree_to_dtree_conversion}. %
%\footnote{For English, we applied Collins' head rules to lexicalize the treebank; for other languages, details are in \S\ref{sec:experiments_spmrl}.} % 
If the original treebank has discontinuous c-trees, we end up with non-projective d-trees or with violations of the nested property, as established in Proposition~\ref{prop:continuous_case}. 
We handle this gracefully by training a non-projective d-parser in the subsequent stage (see \S\ref{sec:trainlabeleddepparser}). 
Note also that this conversion drops the unary nodes (a consequence of Proposition~\ref{prop:weakorder}). These nodes will be recovered in the last stage, as described in \S\ref{sec:unaryrecovery}. 
%By using this encoding technique, every node\footnote{Please note that the pre-terminal nodes in the phrase-structure tree are the PoStags and, therefore, they do not must be saved in the dependency labels.} in Figure~\ref{fig:excontree} is encoded in the dependency tree in Figure~\ref{fig:exdeptree}, except for the unary nodes \emph{ADVP} and \emph{VP}. However, unary nodes can be easily recovered as shows next section.

Since in this paper we are assuming that only an off-the-shelf d-parser is available, we need to convert head-ordered d-trees to plain d-trees. We do so 
by encoding the order information in the dependency labels. 
We tried two different strategies to do that. The first one is a \emphbf{direct encoding}, where we just append suffixes \#1, \#2, etc., as illustrated in Figure~\ref{fig:ctrees_dtrees}. 
A disadvantage is that the number of dependency labels may grow unbounded with the treebank size, since we may encounter complex substructures where the event sequences are long. %, increasing the number of indices to be used.% 
The second strategy is a \emphbf{delta-encoding} scheme where, %, commonly used in data compression. 
rather than writing the absolute indices in the dependency label, we write the \emph{differences} between consecutive indices.% 
\footnote{For example, if $\#1, \#3, \#4$ and $\#2, \#3, \#3, \#5$ are respectively the sequence of indices 
from the head to the left and to the right, we encode these sequences 
as $\#1, \#2, \#1$ and $\#2, \#1, \#0, \#2$ (using 3 distinct indices instead of 5).} % 
We used this strategy for the continuous treebanks only, 
whose d-trees are guaranteed to satisfy the nested property. 
%, ensuring the former sequences are increasing. 

For comparison, we implemented a third strategy replicating the encoding 
proposed by \citet{Hall2008}, which we call \emphbf{H\&N-encoding}. 
%For each arc connecting $h$ to $m$, 
This scheme 
concatenates all the c-nodes' labels in the modifier's spine 
with the attachment position in the head's spine 
(for example, in Figure~\ref{fig:dtree_ctree_conversion}, if the modifier $m_2$ 
has a spine with nodes $X_1,X_2,X_3$, the generated d-label would be 
${\tt X_1|X_2|X_3\#2}$; 
our direct encoding scheme
generates ${\tt Z_2\#2}$ instead). %
%\footnote{\citet{Hall2008} also encodes functional labels, but we leave them out in this experiment.} 
Since their strategy encodes the entire spines into complex arc labels, many such 
labels will be generated, leading to slower runtimes and poorer generalization, as we will see. 

%(plus functional tags) 
%in each dependency label (see \S\ref{sec:related_work}). \afm{we need to explain this better -- HN encoding in English lead to 731 labels; the runtime of the labeler more than doubles.}

%Table~\ref{tab:results_english_exploratory} 
%show the number of labels obtained with these two strategies, for the training portion of the English Penn Treebank. We can see that the delta-encoding scheme yields a smaller set of labels, which is usually beneficial 
%since it makes the task of predicting the labels easier. 
%\afm{27 original labels (non-terminal symbols which are not POS tags), 
%delta yields 69 labels, no-delta yields 75 labels.}
%
%\afm{need a table with exploratory experiments for English, reporting the number of labels, 
%the UAS and LAS of several parsers on the dev-set (including TurboParser without the separate labeler), runtimes, and the final parseval F1 scores on the dev-set.}

For the training portion of the English PTB, 
which contains 27 non-terminal symbols (excluding the POS tags), 
the direct encoding strategy yields 75 labels, %(an average of 2.8 indices per symbol),  
while delta encoding yields 69 labels (averaging 2.6 indices per symbol). 
By contrast, the HN-encoding procedure yields 731 labels,  
more than 10 times as many. 
%This is a slight difference, and we 
We later show (in 
Table~\ref{tab:results_english_exploratory}) 
that delta-encoding leads to a slightly 
higher c-parsing accuracy than direct encoding, %
%\footnote{The benefits of delta-encoding were found sharper for other languages  (\S\ref{sec:experiments_spmrl}).} %; we omit results due to space constraints.} % 
%While this difference is slight, we will later see (in 
%Table~\ref{tab:results_english_exploratory}) 
%that the delta encoding scheme
%leads to a more accurate constituency parsers.%
%\footnote{Similar conclusions were found for the non-English languages; we omit results due to space constraints.} % 
and
that both strategies are considerably more accurate than 
H\&N-encoding. 

\subsection{Training the Labeled Dependency Parser}\label{sec:trainlabeleddepparser}

The next step is to train a labeled d-parser on the converted treebank. If we are doing continuous c-parsing, we train a projective d-parser; otherwise we train a non-projective one. 

%\afm{we also need to train a tagger}

In our experiments, we found it advantageous to 
perform labeled d-parsing in two stages, 
as done by \citet{McDonald2006CoNLL}: 
first, train an unlabeled d-parser; then, train a dependency labeler.%
\footnote{The reason why a two-stage approach is preferable  
is that one-shot d-parsers, for efficiency reasons,  use label features parsimoniously. 
%(this happens for efficiency reasons: for example,  
%in \emph{TurboParser}, only arc-factored label features are implemented, 
%so that the decoder can pick the best label for each candidate arc before decoding the dependency structure).  
However, for our reduction approach, 
the dependency labels are crucial and strongly interdependent, since they jointly encode the constituent structure.} %
Table~\ref{tab:results_english_exploratory} compares this approach against a one-shot strategy, experimenting with various off-the-shelf d-parsers: \emph{MaltParser} \citep{Nivre2007malt}, \emph{MSTParser} \citep{McDonald2005b}, 
\emph{ZPar} \citep{Zhang2011}, and \emph{TurboParser} \citep{Martins2013ACL}, all with the default settings. 
For \emph{TurboParser}, we used 
basic, standard and full models. % (TP-Basic, TP-Standard, and TP-Full, respectively). 
%\afm{is MSTParser first order or second order? projective, right?}\danny{In all parsers, I used the default settings. In MSTParser first order and projective.}

%\afm{zpar is missing from the table}
%The reason why a two-stage approach yields better performance
%is that in the joint model, for efficiency reasons, 
%all features in \emph{TurboParser} involving the labels are arc-factored 
%(which allows to select the best label for each candidate arc before decoding the dependency structure). 
%However, in our approach, 
%the labels are really crucial since they encode the constituency structure, 
%being strongly interdependent. 
Our separate d-labeler receives as input a backbone d-structure and predicts a label for each arc. 
For each head $h$, we decode the modifiers' labels independently from the other heads, using a simple sequence model, 
which contains features of the form $\vectsymb{\phi}(h,m,\ell)$ and $\vectsymb{\phi}(h,m,m',\ell,\ell')$, where $m$ and $m'$ are two consecutive modifiers (either on the same side or on opposite sides of the head) and $\ell$ and $\ell'$ are their labels. 
We used the same arc label features $\vectsymb{\phi}(h,m,\ell)$ as 
\emph{TurboParser}. % plus the following additional ones for 
For $\vectsymb{\phi}(h,m,m',\ell,\ell')$, 
we use the POS triplet $\langle p_h, p_m, p_{m'} \rangle$, 
plus unilexical versions of this triplet, 
where each of the three POS is replaced by the word form. 
%\begin{itemizesquish}
%\item The POS triplet $\langle p_h, p_m, p_{m'} \rangle$;
%\item Unilexical features $\langle w_h, p_m, p_{m'} \rangle$, $\langle p_h, w_m, p_{m'} \rangle$, $\langle p_h, p_m, w_{m'} \rangle$.
%\end{itemizesquish}
Both features are conjoined with the label pair $\ell$ and $\ell'$. 
Decoding under this model can be done by running the Viterbi algorithm independently for each head. The runtime is almost negligible compared with the time to parse: 
it took 2.1 seconds to process PTB~\S{22}, 
a fraction of about 5\% of the total runtime. 
%
%it processes 18,676 tokens \afm{sent/sec?} per second in English PTB~\S{22}, while the parser processes 976 tokens per second.% The speed of the joint system is 
%%912 tokens per second.%
%\footnote{Measured on a Intel Xeon processor @2.30GHz. All reported runtimes include the time spent on pruning and feature evaluation, besides decoding.} % 
%\afm{report hardware here? maybe we should put the table back. Let's wait to see how it goes in terms of space.}
% (see  Table~\ref{tab:runtime_components}). 

\begin{table*}
\begin{center}
%\small
\begin{tabular}{@{\hskip 0pt}l@{\hskip 5pt}c@{\hskip 10pt}c@{\hskip 10pt}c@{\hskip 1pt}r@{\hskip 0pt}}
\noalign{\smallskip}\hline\noalign{\smallskip}
{\bf Dependency Parser} & {\bf UAS} & {\bf LAS} & {\bf F$_1$}  & {\bf \#Toks/s.}\\
%\noalign{\smallskip}
\hline\noalign{\smallskip}
MaltParser & 90.93 & 88.95  & 86.87 & 5,392 \\
MSTParser & 92.17 & 89.86 & 87.93 & 363 \\
ZPar & 92.93 & 91.28  &  89.50 & 1,022 \\
TP-Basic & 92.13 & 90.23 &  87.63 & 2,585 \\
TP-Standard & 93.55 & 91.58 & 90.41 & 1,658 \\
TP-Full & 93.70 & 91.70 &  90.53 & 959 \\
\hline\noalign{\smallskip}
TP-Full + labeler, H\&N encoding & 93.80  & 87.86 & 89.39 & 871  \\
TP-Full + labeler, direct encoding & 93.80  & 91.99 & 90.89  & 912 \\
{\bf TP-Full + labeler, delta encoding} & {\bf 93.80}  & {\bf 92.00} &  {\bf 90.94} & 912 \\
%\noalign{\smallskip}
\hline
\end{tabular}
\caption{Results on English PTB~\S{22} achieved by various d-parsers and encoding strategies. For dependencies, we report unlabeled/labeled attachment scores (UAS/LAS), excluding punctuation. For constituents, we show F$_1$-scores (without punctuation and root nodes), as provided by \url{EVALB} \citep{Black1992}. 
We report total parsing speeds in tokens per second (including time spent on pruning, decoding, and feature evaluation), measured on a Intel Xeon processor @2.30GHz.
%\afm{maybe we can skip UAS -- it's not very informative, and we would need to explain 
%why the UAS of full is different from full w/o labeler -- let's keep for now, later we decide.}
\label{tab:results_english_exploratory}}
\end{center}
\end{table*}

\subsection{Decoding into Unaryless Constituents}

After training the labeled d-parser, we can run it on the test data. 
Then, we need to convert the predicted d-tree into a c-tree without unaries. 

To accomplish this step, we first need to recover, for each head $h$, the weak order of its modifiers $M_h$. We do this by looking at the predicted dependency labels, extract the event indices $j$, and use them to build and sort the equivalent classes $\{\bar{M}_h^j\}_{j=1}^J$. If two modifiers have the same index $j$, we force them to have consistent labels (by always choosing the label of the modifier which is the closest to the head). For continuous c-parsing, we also 
%perform the necessary corrections to ensure the weak order respects the nesting property. % 
%{Namely, we 
decrease the index $j$ of the modifier closer to the head as much as necessary to make sure 
%this property 
that the nesting property
holds. In PTB~\S{22}, these corrections were necessary only for 0.6\% of the tokens. 
Having done this, we use Algorithm~\ref{alg:dtree_to_ctree_conversion} to obtain a predicted c-tree without unary nodes. 
%At this point, we can measure bracketing scores without unaries. \afm{not sure we need to report these, so maybe just skip this sentence}

%From a dependency encoded structure we can directly recover the original constituency tree without unary nodes by following a simple bottom-up strategy. Starting from the leaves of the dependency tree, we substitute all those dependency arcs with the same dependency label by a subtree compounded by the head and its dependents as children nodes and the node label saved in the dependency label as the parent node. To keep the same structure as the original constituency tree, those arcs with a lower distance (index) in the dependency label are substituted sooner than those that will be located in higher positions in the resulting constituency tree. In the example, to obtain the phrase-structure tree without unary nodes in Figure~\ref{fig:excontree2}, we have to apply the following substitutions in the dependency tree in Figure~\ref{fig:exdeptree} (explain it or draw an example?):

\subsection{Recovery of Unary Nodes}\label{sec:unaryrecovery}

Finally, the last stage is to recover the unary nodes.  
Given a unaryless c-tree as input, we 
predict unaries by running independent multi-class classifiers at each node in the tree (a simple unstructured task). 
Each class is either {\sc null} (in which case no unary node is appended to the current node) 
or a concatenation 
%\afm{this example is not very enlightning... also at this point it may not be clear to the reader that we 
%can have a sequence of unaries instead of just one.} %
of 
unary node labels (e.g., {\tt S->ADJP} for a node {\tt JJ}); 
we obtained 64 classes by processing the training 
sections of the PTB, the fraction of unary nodes 
being about 11\% of the total number of non-terminal nodes. %
%\footnote{This is much larger for some languages in the SPMRL datasets, such as Hebrew and Polish, where it surpasses 60\%.} % 
To reduce complexity, for each node symbol we 
only consider classes that have been observed for that symbol in 
the training data. 
In PTB~\S{22}, we obtained an average of 9.9 candidate labels 
per node occurrence. 
%\afm{average number of label candidates per node; total number of labels; say that we do basic pruning}

% English Dev: 64904 nodes, 7417 are non-unaries,
% 9.9 candidate labels per node, 
% 64 labels (unary spines).

These classifiers are trained on the original c-treebank, stripping off unary nodes and trained to recover those nodes. 
We used the following features (conjoined with the class and with a flag indicating if the node is a pre-terminal):
\begin{itemizesquish}
\item The production rules above and beneath the node (\emph{e.g.}, {\tt S->\underline{NP} VP} and {\tt \underline{NP}->DT NN}); % if the node is {\tt NP});
\item The node's label, alone and conjoined with the parent's label or the left/right sibling's label;
\item The leftmost and rightmost word/lemma/POS tag/morpho-syntactic tags in the node's yield;
\item If the left/right node is a pre-terminal, the word/lemma/morpho-syntactic tags beneath.
\end{itemizesquish} 
This is a relatively easy task: when 
gold unaryless c-trees provided as input, we obtain 
an EVALB F$_1$-score of 99.43\%. 
%The F$_1$ score for unary prediction on the 
%English PTB~\S{22} dataset (with gold unaryless c-trees provided as input) is 74.4\%, 
%which leads to a EVALB F$_1$-score of 99.43\%. 
This large figure is explained by the fact that 
there are few unary nodes in the gold data, 
so 
this module does not impact the final parser as much as the d-parser. 
Being a lightweight unstructured task, this step 
took only 0.7 seconds to run on PTB~\S{22}, 
representing %a tiny fraction of the total runtime: 
%it took 0.7 seconds to run on PTB~\S{22}, 
a tiny fraction (less than 2\%) of the total runtime. 
%the speed is 54,655 tokens per second \afm{sent/sec?}, more than 50 times faster than the d-parser. 

Table~\ref{tab:results_english_exploratory} reports the %final 
accuracies obtained with the d-parser followed by the unary predictor. 
Since two-stage TP-Full with delta-encoding is the best strategy, 
we use this configuration in the subsequent experiments. 

%The $F_1$ scores are also fairly high; however, only ??\% of the nodes are unary in the gold data, therefore this module does not impact the final parser as much as the dependency parser. 
%
%Table~\ref{tab:results_english_exploratory} reports the accuracies obtained with this unary predictor. \afm{actually composing the entire system...}

%\afm{maybe we should report non-trained unary predictors that always predict the most common unary spine.}

%To recover unary nodes in the dependency-to-constituent transformation, a model must be trained. Thus, we can recover the whole constituent tree in Figure~\ref{fig:excontree}. Maybe unary recovering has an own section??. 

\section{Experiments}\label{sec:experiments}

%\subsection{Set-up}

%In order to evaluate the performance of our system, we conduct experiments in both continuous and discontinuous constituency treebanks. 

We now compare %evaluate 
%the performance of 
our reduction-based parsers 
with other state-of-the-art c-parsers 
in a variety of treebanks, both continuous and discontinuous. 

\subsection{Results on the English~PTB}\label{sec:experiments_ptb}

%We experimented on the 
%English Wall Street Journal corpus of the  PennTreebank~\citep{penn} (PTB). We use sections 2-21 as training data, section 22 for system development and section 23 for final assessment.   
%The POS tagging was performed by \emph{TurboTagger}~\citep{Martins2013ACL}, applying a 10-fold jackknifing on training sets.  

Table \ref{tab:result1} shows the accuracies and speeds %achieved by our system in comparison to state-of-the-art c-parsers 
on the English PTB~\S{23}. %
%\footnote{Scores were computed with  \url{EVALB} ({\url{http://nlp.cs.nyu.edu/evalb/}}) ~\citep{Black1992} with the \url{COLLINS.prm} parameter file, ignoring the root node and punctuation (excluding brackets). Runtimes were measured on a computer with 115GB of memory and Intel Xeon E5-2670 v3 2.30GHz CPU.} % 
We can see that our simple reduction-based c-parser 
surpasses the three Stanford parsers \cite[and Stanford Shift-Reduce]{Klein2003,Socher2013},  
and 
is on par with the 
%competitive supervised systems, such as the 
Berkeley parser 
\citep{Petrov2007NAACL}, while being more than 5 times faster. 
The best %speed/accuracy trade-off is 
%achieved by the %
supervised competitor is the 
recent shift-reduce parser 
of \citet{Zhu2013}, which achieves slightly better accuracy and speed. %similar accuracies and is slightly faster than ours. 
Our technique has the advantage of being flexible: since the time for d-parsing is the dominating factor (see~\S\ref{sec:unaryrecovery}), 
plugging a faster d-parser automatically yields a faster c-parser. 
Orthogonal techniques, 
such as semi-supervised training and 
reranking, can also %equally 
be applied 
to our parser to boost its performance. 
%While reranking and semi-supervised systems %based on reranking or trained in a 
%%semi-supervised manner 
%achieve higher accuracies, 
%this is an orthogonal aspect, since the same 
%techniques can be applied to our parser. 

%As we can see, our approach outperforms well-known systems such as \citet{Klein2003} and \citet{Petrov2007NAACL} 
%\afm{not} and is four and three times faster, respectively. In addition, it is competitive in speed and accuracy to \citet{Zhu2013}.
%\afm{need to say a little bit more}

\begin{table}
\begin{center}
\small
\begin{tabular}{@{\hskip 0pt}l@{\hskip 8pt}c@{\hskip 8pt}c@{\hskip 8pt}c@{\hskip 1pt}r@{\hskip 0pt}}
\hline\noalign{\smallskip}
{\bf Parser}
& {\bf LR} & {\bf LP} 
& {\bf F1} & {\bf \#Toks/s.} \\
\hline\noalign{\smallskip}
\citet{Charniak2000} & 89.5 & 89.9 & 89.5 & -- \\
\citet{Klein2003} & 85.3 & 86.5 & 85.9 & 143   \\
\citet{Petrov2007NAACL}  & 90.0 & 90.3 & 90.1 & 169   \\
\citet{Carreras2008} & 90.7 & 91.4 & 91.1 & --   \\
\citet{Zhu2013}& 90.3 & 90.6 & 90.4 &  1,290  \\
Stanford Shift-Reduce (2014) & 89.1 & 89.1 & 89.1 &  655  \\
\citet{Hall2014} & 88.4 & 88.8 & 88.6 &  12  \\
{\bf This work} & {89.9} & {90.4} & {90.2} & 957 \\
\hline\noalign{\smallskip}
\citet{Charniak2005}$^*$ &  91.2 &  91.8 & 91.5 & 84   \\
\citet{Socher2013}$^*$  & 89.1 & 89.7 & 89.4 & 70   \\
%\citet{Charniak2000} & 89.5 & 89.9 & 89.5 & -- \\
%\citet{Klein2003} & 85.3 & 86.5 & 85.9 & 6.1   \\
%\citet{Petrov2007NAACL}  & 90.0 & 90.3 & 90.1 & 7.2   \\
%\citet{Carreras2008} & 90.7 & 91.4 & 91.1 & --   \\
%\citet{Zhu2013}& 90.3 & 90.6 & 90.4 &  55.0  \\
%Stanford Shift-Reduce (2014) & 89.1 & 89.1 & 89.1 &  27.9  \\
%\citet{Hall2014} & 88.4 & 88.8 & 88.6 &  0.5  \\
%{\bf This work} & {89.9} & {90.4} & {90.2} & {40.8} \\
%\hline\noalign{\smallskip}
%\citet{Charniak2005}$^*$ &  91.2 &  91.8 & 91.5 & 3.6   \\
%\citet{Socher2013}$^*$  & 89.1 & 89.7 & 89.4 & 3.0   \\
\hline
\end{tabular}
\caption{Results on the English PTB~\S{23}. All 
systems reporting runtimes 
%systems being compared 
were run on the same machine. % except $^\dagger$, whose accuracies are the ones reported in the cited papers. 
%ZPar and Stanford Shift-Reduce are two different implementations of \citet{Zhu2013}. 
Marked as $^*$ are reranking and semi-supervised c-parsers. 
%Berkeley is \citet{Petrov2007NAACL}, 
%Stanford-PCFG is \citet{Klein2003},
%Stanford-RNN is \citet{Socher2013},
%CJ-Reranking is \citet{Charniak2005}.
\label{tab:result1}}
\end{center}
\end{table}

%\citet{Zhu2013} is implemented in two different frameworks: ZPar and Stanford.

\subsection{Results on the SPMRL Datasets}\label{sec:experiments_spmrl}

We experimented with 
datasets for eight morphologically rich languages, from the 
SPMRL14 shared task~\citep{Seddah2014}.%
\footnote{We left out the Arabic dataset for licensing reasons.} %
We used the official training, development and test sets with the provided predicted POS tags, 
and different lexicalization rules for each language. For French and German we used the head rules detailed in \citet{Johansen2004} and \citet{Rehbein2009}, respectively. 
For Basque, Hungarian and Korean, we always take the rightmost modifier as head-child node. For Hebrew and Polish we use the leftmost modifier instead. For Swedish we 
induce head rules from the provided dependency treebank, as described in~\citet{Versley2014Arxiv}. These choices were based on dev-set experiments.

Table \ref{tab:result2} shows the results. %
%\footnote{Evaluation was carried out with the provided  \url{EVALB_SPMRL} (\url{pauillac.inria.fr/~seddah/evalb_spmrl2013.tar.gz}) implementation that takes into account for evaluation all tokens except root nodes.} % 
For all languages except French, 
our system outperforms the Berkeley parser \citep{Petrov2007NAACL}, 
with or without prescribed POS tags.  
Our average F$_1$-scores are superior to the 
best single parser%
\footnote{By ``single parser'' we mean a system which does not use ensemble or reranking techniques.} %
participating in the shared task \citep{Crabbe2014}, 
and to the system of \citet{Hall2014}, 
achieving the best results for 4 out of 8 languages. 

%and is competitive with the best single parser%
%\footnote{By ``single parser'' we mean a system which does not use ensemble or reranking techniques.} %
%that participated in the shared task \citep{Crabbe2014}, outperforming this in Basque, German, Hebrew and Swedish. We also achieve the best averaged F$_1$-score, improving over the system of \citet{Hall2014} on five out of eight languages.

\begin{table*}
\begin{center}
\begin{small}
\begin{tabular}{@{\hskip 0pt}l@{\hskip 8pt}c@{\hskip 8pt}c@{\hskip 8pt}c@{\hskip 8pt}c@{\hskip 8pt}c@{\hskip 8pt}c@{\hskip 8pt}c@{\hskip 8pt}c@{\hskip 8pt}c@{\hskip 0pt}}
%\cline{2-10}\noalign{\smallskip}
%& \multicolumn{9}{c}{Language} \\
\hline\noalign{\smallskip}
{\bf Parser} & {\bf Basque} & {\bf French} & {\bf German} & {\bf Hebrew} & {\bf Hungar.} & {\bf Korean} & {\bf Polish} & {\bf Swedish} & {\bf Avg.}  \\
\hline\noalign{\smallskip}
%\citet{Petrov2007NAACL} 
Berkeley & 70.50 & \textbf{80.38} & 78.30 & 86.96 & 81.62 & 71.42 & 79.23 & 79.19 & 78.45 \\
Berkeley Tagged %\citet{Petrov2007NAACL}  Tagged 
& 74.74 & 79.76 & 78.28 & 85.42 & 85.22 & 78.56 & 86.75 & 80.64 & 81.17  \\
\citet{Hall2014} & 83.39 & 79.70 & 78.43 & 87.18 & \textbf{88.25} & \textbf{80.18} & 90.66 & 82.00 & 83.72 \\
\citet{Crabbe2014} & 85.35 & 79.68 & 77.15 & 86.19 & 87.51 & 79.35 & \textbf{91.60} & 82.72 & 83.69  \\
%This work \afm{old} & 85.08 & 78.94 & 78.48 & 88.22 & 88.21 & 79.34 & 91.16 & 82.63 & 84.01  \\
{\bf This work} & \textbf{85.90} & 78.75 & \textbf{78.66} & \textbf{88.97} & 88.16 & 79.28 & 91.20 & \textbf{82.80} & \textbf{84.22}  \\
\hline\noalign{\smallskip}
\citet{Bjorkelund2014} & 88.24 & 82.53 & 81.66 & 89.80 & 91.72 & 83.81 & 90.50 & 85.50 & 86.72  \\
\hline
\end{tabular}
\caption{F$_1$-scores  on eight treebanks of the SPMRL14 shared task, computed with the provided {\tt EVALB\_SPMRL} tool, 
(\url{pauillac.inria.fr/~seddah/evalb_spmrl2013.tar.gz}) 
which takes into account all tokens except root nodes. 
Berkeley Tagged is a version of \citet{Petrov2007NAACL} using the predicted POS tags provided by the organizers. \citet{Crabbe2014} is the best non-reranking system in the shared task, and \citet{Bjorkelund2014} the ensemble and reranking-based system which won the official task. %\citet{Hall2014} did not participate in the SPMRL14 shared task.
\label{tab:result2}}
\end{small}
\end{center}
\end{table*}

\subsection{Results on the Discontinuous Treebanks}\label{sec:experiments_disco}

Finally, we experimented on two widely-used discontinuous German treebanks: 
TIGER \citep{tiger} and NEGRA \citep{negra}. For the former, we used two different splits: TIGER-SPMLR, provided in the SPMRL14 shared task; and TIGER-H\&N, used by~\citet{Hall2008}. For NEGRA, we used the standard splits. % (sentences 1--18,602 for training, 19,603--20,602 for validation, 18,603--19,602 for testing). 
% \begin{itemizesquish}
% \item TIGER treebank~\citep{tiger}. For comparison with previous work, we used two different splits: TIGER-SPMLR, provided in the SPMRL14 shared task; and TIGER-H\&N, used by~\citet{Hall2008}. 
% \item NEGRA treebank~\citep{negra}. We apply the commonly-used split (sentences 1--18,602 for training, 19,603--20,602 for validation, 18,603--19,602 for testing). 
% \end{itemizesquish}
In these experiments, we skipped the unary recovery stage, since very few unary nodes exist in the data.%
\footnote{NEGRA has no unaries; for the TIGER-SPMRL and H\&N dev-sets, the fraction of unaries is 1.45\% and 1.01\%.} % 
For the TIGER-SPMRL dataset, we used the predicted POS tags provided in the shared task. For TIGER-H\&N and NEGRA, we predicted POS tags with \emph{TurboTagger}.   
The treebanks were lexicalized using the head-rule sets of \citet{Rehbein2009}. 
For comparison to related work, a sentence length cut-off of 30, 40 and 70 was applied during the evaluation. 

\if 0
The evaluation metrics were obtained with the Disco-DOP evaluator.%
\footnote{\url{http://discodop.readthedocs.org}} %
For the assessment of the TIGER-H\&N and NEGRA treebanks, the root nodes and all punctuation (including brackets) are not considered in the final score; and, for the TIGER-SPMRL dataset, only root nodes are removed from the evaluation. 
\fi 

Table~\ref{tab:result4} shows %--\ref{tab:result5} show %the results. 
%We observe 
that our approach outperforms  all the competitors considerably, %on the two German treebanks, 
achieving state-of-the-art accuracies for both datasets. 
The best competitor, 
\citet{Cranenburgh2013}, 
is more than 3 points behind, both in TIGER-H\&N and in NEGRA. 
Our reduction-based parsers are also much faster: 
\citet{Cranenburgh2013} report 3 hours to parse NEGRA with $L<40$. 
Our system parses all NEGRA sentences 
(regardless of length) in 27.1 seconds, 
which corresponds to a rate of 618 toks/s. %37 sent./sec.
This approaches the speed of the 
easy-first system of \citet{Versley2014SPMRL}, 
who 
reports runtimes in the range 670--920 toks/s., 
but is much less accurate. %40--55 sent./sec. (but is much less accurate). 

%\danny{Speed on test sets: TIGER-SPMRL(92004 tokens, 5000 sent.) 28.92 sent./s., TIGER-H\&N (88515 tokens, 5047 sent.) 38.64 and NEGRA (16742 tokens, 1000 sent.) 36.93 sent./s.}

\begin{table*}
\begin{center}
%\begin{scriptsize}
\begin{tabular}{lcccccc}
\noalign{\smallskip}\hline\noalign{\smallskip}
\multicolumn{7}{c}{\bf TIGER-SPMRL}\\
\noalign{\smallskip}\hline\noalign{\smallskip}
 & \multicolumn{2}{c}{$L\leq40$} & \multicolumn{2}{c}{$L\leq70$} &  \multicolumn{2}{c}{all} \\
{\bf Parser}  & F$_1$ & EX & F$_1$ & EX & F$_1$ & EX \\
\noalign{\smallskip}\hline\noalign{\smallskip}
\citet{Versley2014Arxiv}, \emph{gold} & 78.34 & 42.78 & 76.46 & 41.05 & 76.11 & 40.94 \\
\textbf{This work, \emph{gold}} & \textbf{82.56} & \textbf{45.25} &\textbf{80.98} & \textbf{43.44} &  \textbf{80.62} & \textbf{43.32}    \\
\hline\noalign{\smallskip}
\citet{Versley2014Arxiv}, \emph{pred} & -- & -- & 73.90 & 37.00 & -- & -- \\
\textbf{This work, \emph{pred}} & \textbf{79.57} & \textbf{40.39} & \textbf{77.72} & \textbf{38.75} &  \textbf{77.32} & \textbf{38.64}    \\
\noalign{\smallskip}\hline
%\end{tabular}
%\begin{tabular}{lcccccc}
\noalign{\smallskip}\hline\noalign{\smallskip}
\multicolumn{7}{c}{\bf TIGER-H\&N}\\
\noalign{\smallskip}\hline\noalign{\smallskip}
 & \multicolumn{2}{c}{$L\leq30$} &  \multicolumn{2}{c}{$L\leq40$} & \multicolumn{2}{c}{all} \\
{\bf Parser} & F$_1$ & EX & F$_1$ & EX & F$_1$ & EX \\
\noalign{\smallskip}\hline\noalign{\smallskip}
\citet{Hall2008}, \emph{gold} & -- & -- & 79.93 & 37.78 & -- & -- \\
\citet{Versley2014SPMRL}, \emph{gold} & 76.47 & 40.61 & 74.23 & 37.32 & -- & -- \\
 \textbf{This work, \emph{gold}} & \textbf{86.63} & \textbf{54.88} &  \textbf{85.53} & \textbf{51.21}  &  \textbf{84.22} & \textbf{49.63}  \\
\hline\noalign{\smallskip}
 \citet{Hall2008}, \emph{pred} & -- & -- & 75.33 & 32.63 & -- & -- \\
\citet{Cranenburgh2013}, \emph{pred} & -- & -- & 78.8-- & 40.8-- & -- & -- \\
 \textbf{This work, \emph{pred}} & \textbf{83.94} & \textbf{49.54} &  \textbf{82.57} & \textbf{45.93}  &  \textbf{81.12} & \textbf{44.48}  \\
  \noalign{\smallskip}\hline
  \noalign{\smallskip}\hline\noalign{\smallskip}
\multicolumn{7}{c}{\bf NEGRA}\\
\noalign{\smallskip}\hline\noalign{\smallskip}
 & \multicolumn{2}{c}{$L\leq30$} &  \multicolumn{2}{c}{$L\leq40$} & \multicolumn{2}{c}{all} \\
{\bf Parser}  & F$_1$ & EX & F$_1$ & EX & F$_1$ & EX \\
\noalign{\smallskip}\hline\noalign{\smallskip}
\citet{Maier2012}, \emph{gold} & 74.5-- & -- & -- & -- & -- & -- \\
\citet{Cranenburgh2012}, \emph{gold} &  -- & -- & 72.33 & 33.16 & 71.08 & 32.10 \\
\citet{Kallmeyer2013}, \emph{gold} & 75.75 & -- & -- & -- & -- & -- \\
\citet{Cranenburgh2013}, \emph{gold} & -- & -- & 76.8-- & 40.5-- & -- & -- \\
 \textbf{This work, \emph{gold}} & \textbf{82.56} & \textbf{52.13} &  \textbf{81.08} & \textbf{48.04}  &  \textbf{80.52} & \textbf{46.70} \\
\hline\noalign{\smallskip}
\citet{Cranenburgh2013}, \emph{pred} & -- & -- & 74.8-- & 38.7-- & -- & -- \\
 \textbf{This work, \emph{pred}} & \textbf{79.63} & \textbf{48.43} &  \textbf{77.93} & \textbf{44.83}  &  \textbf{76.95} & \textbf{43.50}  \\
\noalign{\smallskip}\hline
\end{tabular}
\caption{Results on TIGER and NEGRA test partitions, with \emph{gold} and \emph{predicted} POS tags. Shown are F$_1$ and exact match scores (EX), computed 
with the Disco-DOP evaluator (\url{discodop.readthedocs.org}), 
ignoring root nodes and, for TIGER-H\&N and NEGRA, punctuation tokens.}
%\end{scriptsize}
\end{center}
\label{tab:result4}
\end{table*}

\section{Related Work}\label{sec:related_work}

Conversions between constituents and dependencies  
have been considered by \citet{Marneffe2006} in the forward direction, 
and by \citet{Collins1999ACL} and \citet{Xia2001} in the backward direction, 
toward the construction of 
multi-representational treebanks \citep{Xia2008}.  
This prior work aimed at linguistically sound conversions,  
involving grammar-specific transformation rules to handle the kind of ambiguities expressed in Figure~\ref{fig:ctrees_ambiguity}. 
% 
%Making a constituency parser produce dependencies 
%is easy, by 
%using head percolation tables and simple transformation rules 
%\citep{Collins1999,Yamada2003,Marneffe2006}.  
%Prior work aimed at linguistically sound conversions \citep{Collins1999ACL,Xia2001} and 
%the construction of 
%multi-representational treebanks \citep{Xia2008}, both involving grammar-specific transformation rules to handle the kind of ambiguities expressed in Figure~\ref{fig:ctrees_ambiguity}. 
Our work differs in that 
we are not concerned about 
the linguistic plausibility of our conversions, 
but only with the formal aspects that underlie 
the two representations. 

The work most related to ours is \citet{Hall2008}, 
who also convert dependencies to constituents   
to prototype a c-parser for German. 
Their encoding strategy is compared to ours in \S\ref{sec:encoding}:  
they encode the entire spines 
into the dependency labels, 
which become rather complex and numerous. 
A similar strategy has been used by \citet{Versley2014SPMRL} for discontinuous c-parsing. 
Both are largely outperformed by our system, as shown in \S\ref{sec:experiments_disco}. 
The crucial difference is that we encode only the top node's label and its position in the spine---besides being 
a much lighter representation, ours has an interpretation as a weak ordering,  leading to the isomorphisms expressed in 
Propositions~\ref{prop:strictorder}--%
\ref{prop:continuous_case}. 

Joint constituent and dependency parsing have been 
tackled by \citet{Carreras2008} and \citet{Rush2010}, but the resulting parsers, while accurate, are more expensive than a single c-parser. %make runtimes even slower.   
Very recently, 
\citet{Kong2015} proposed a much cheaper pipeline in which 
d-parsing is performed first, followed 
by a c-parser constrained to be consistent with the predicted d-structure. 
%Our approach is simpler in that 
Our work differs in which 
we do not need to run a 
c-parser in the second stage---instead, 
the d-parser already stores constituent information in the arc labels, and the only necessary post-processing is to recover unary nodes. 
Another advantage of our method is that it can be readily used for discontinuous parsing, while their %\citet{Kong2015}'s 
constrained CKY algorithm 
can only produce continuous parses.

%
%Other research works developed more complex techniques to achieve the same purpose as those by \citet{Hall2008}, by \citet{Versley2014SPMRL} and by Kong, Rush and Smith. \afm{add ref}
% 
%\citet{Hall2008} and \citet{Versley2014SPMRL} encode constituency trees in dependency structures as we do, but they use labels that encode more information, especially \citet{Hall2008}. 
% 
% 
% 
%\citet{Hall2008} and \citet{Versley2014SPMRL} use the pseudo-projective algorithm, therefore, they are not using purely non-projective parsers. Ours uses a pure non-projective algorithm and, maybe, this is the reason why our technique achieves a better performance.

\section{Conclusion}
We proposed a reduction technique that allows to 
implement a constituent parser when only a dependency parser is given. 
The technique is applicable to any dependency parser, regardless its nature or kind. 
This reduction was accomplished by endowing 
dependency trees with a weak order relation, 
and showing that the resulting class of 
head-ordered dependency trees is isomorphic to 
constituent trees. 
We have shown empirically that the proposed reduction, while simple, leads to 
highly-competitive constituent parsers 
for English and for eight morphologically rich languages; 
and that it outperforms the current state of the art in discontinuous parsing of German.

\section*{Acknowledgments}
We thank Slav Petrov, Lingpeng Kong and Carlos G\'omez-Rodr\'iguez for their comments and suggestions. 
This research has been partially funded by the Spanish Ministry of Economy and Competitiveness and FEDER (project TIN2010-18552-C03-01), 
Ministry of Education (FPU Grant Program) and Xunta de Galicia (projects R2014/029 and R2014/034). 
A.~M.~was supported  
by the EU/FEDER programme, 
QREN/POR Lisboa (Portugal), under the 
Intelligo 
project 
(contract 2012/24803), 
and by the FCT grant UID/EEA/50008/2013. 

%\section{Garbage Collector}
%
%\afm{The overall complexity of context-free parsing is $O(|G|L^3)$, where $|G|$ is the size of the grammar. Parse time
%with large, high-accuracy grammars depends
%primarily on the ``grammar constant''
%rather than the cubic factor in the length
%of the string.}
%
%\afm{the spines consist
%of a lexical anchor together with a series of unary
%projections, which usually correspond to different
%X-bar levels associated with the anchor}

%\appendix

\bibliographystyle{acl}
\bibliography{acl2015}

\end{document}